\documentclass[10pt,twocolumn,letterpaper]{article}

\usepackage{cvpr}      
\usepackage{mathtools}
\usepackage{fontawesome}
\usepackage{amsthm}
\definecolor{cvprblue}{rgb}{0.21,0.49,0.74}
\usepackage[pagebackref,breaklinks,colorlinks,allcolors=cvprblue,bookmarks=false]{hyperref}
\newcommand{\lorenzo}[1]{}

\newcommand{\inv}[1]{#1^{\raisebox{.2ex}{$\scriptscriptstyle-1$}}}

\newcommand{\msqrt}[1]{#1^{\raisebox{.2ex}{$\scriptscriptstyle\frac{1}{2}$}}}
\newcommand{\vct}[1]{\boldsymbol{#1}\xspace}
\newcommand{\mat}[1]{\mathtt{#1}\xspace}
\newcommand{\vx}{\vct x}

\newcommand{\vmu}{\vct \mu}
\newcommand{\vE}{\mat\Sigma}
\newcommand{\viE}{\inv{\mat\Sigma}}

\newcommand{\fDray}{\mathcal D_{\text{ray}}}
\newcommand{\fDgs}{\mathcal D_{\text{gs}}}
\newcommand{\fD}{\mathcal D}

\newcommand{\fH}{\mathcal H}

\newcommand{\vmuh}{\hat{\vmu}}

\newcommand{\vv}{\vct v}

\newcommand{\vRv}{\mat R_{\vmuh\leftarrow\vv}}

\newcommand{\bij}{\mathrel{\mathop{\rightleftarrows}}} 

\newcommand{\myparagraph}[1]{

\vspace{3pt}\noindent\textbf{#1.}\xspace}

\newtheorem{proposition}{Proposition}[section]
\AtEndPreamble{
    \crefname{proposition}{Prop.}{Props.}
}

\newcommand\blfootnote[1]{%
  \begingroup
  \renewcommand\thefootnote{}\footnote{#1}%
  \addtocounter{footnote}{-1}%
  \endgroup
}
\def\paperID{13598} 
\def\confName{CVPR}
\def\confYear{2025}

\title{
Hardware-Rasterized Ray-Based Gaussian Splatting
\vspace{-5pt}
}

\author{Samuel Rota Bul\`o, Nemanja Bartolovic, Lorenzo Porzi, Peter Kontschieder\\
Meta Reality Labs, Z\"urich
\vspace{-10pt}
}

\begin{document}
\maketitle
\begin{abstract}
We present a novel, hardware-rasterized rendering approach for ray-based 3D Gaussian Splatting (RayGS), obtaining both fast and high-quality results for novel view synthesis. Our work contains a mathematically rigorous and geometrically intuitive derivation about how to efficiently estimate all relevant quantities for rendering RayGS models, structured with respect to standard hardware rasterization shaders. 
Our solution is the first enabling rendering RayGS models at sufficiently high frame rates to support quality-sensitive applications like Virtual and Mixed Reality. Our second contribution enables alias-free rendering for RayGS, by addressing MIP-related issues arising when rendering diverging scales during training and testing. We demonstrate significant performance gains, across different benchmark scenes, while retaining state-of-the-art appearance quality of RayGS.\blfootnote{\faGithub\xspace \url{https://github.com/facebookresearch/vkraygs}} 
\end{abstract}

\vspace{-10pt}
\section{Introduction}
The advent of recently introduced image-based reconstruction methods like Neural Radiance Fields (NeRFs)~\cite{mildenhall2020nerf} and 3D Gaussian Splatting (3DGS)~\cite{kerbl3Dgaussians} has paved the way for a new era of photorealistic novel view synthesis in Virtual and Mixed Reality applications. Radiance fields can capture subtle nuances of real-world scenes, including fine-grained texture details, complex lighting phenomena, and transparent surfaces. In particular, 3DGS has a number of interesting properties including interpretability, flexibility, and efficiency for both training and real-time rendering. 

Once the scene models are reconstructed, fast and real-time rendering capabilities are required for applications like novel view synthesis. While the original 3DGS paper is already significantly faster compared to NeRFs, further works have addressed improving rendering speed~\cite{navaneet2023compact3d,fast-gauss,feng2024flashgs}. However, many of their target applications still center on rendering for 2D consumption or on small (mobile) screens, which allows for certain glitches in terms of rendering artifacts and overall quality. In contrast, we observed that VR applications set the bar significantly higher when it comes to rendering quality, otherwise breaking the immersion and thus the overall experience. 

A number of works have contributed to further improving 3DGS’ reconstruction quality, \eg, by using a better densification procedure~\cite{rota2024revising}, removing popping artifacts~\cite{radl2024stopthepop}, or leveraging a ray casting based approach for computing ray-Gaussian intersections~\cite{yu2024gaussian,hahlbohm2024efficient,huang20242d} rather than using the originally proposed splatting formulation. Many of these improvements are complementary in nature, but particularly ray-tracing based volume rendering of 3D Gaussians (RayGS)~\cite{huang20242d} has shown superior quality. The quality gain is mostly due to eliminating some of the approximations needed in traditional 3DGS, however, this comes at increased computational costs making it unsuitable for high frame rate applications based on consumer-grade hardware.

In our work we propose a novel and substantially faster renderer based on hardware-rasterization, while retaining the high quality of ray-based Gaussian Splatting. Hardware rasterization pipelines have been successfully demonstrated for standard 3DGS~\cite{vkgs}, enabling cross-platform usage based on using standard graphics pipeline components. For standard 3DGS, vertex shaders were used to determine the support area of each Gaussian primitive by means of the quad enclosing the minimum area of the corresponding, projected ellipse on the image plane. The fragment shader then computes the primitive’s opacity information, given the positional information computed by the vertex shader. Providing the analogous quantities for the RayGS case is non-trivial, and forms the core contribution of our paper. 
We propose a solution that yields the smallest enclosing quads in 3D space, which turns out to be approximately as fast as the hardware-rasterized variant for standard 3DGS, while retaining the higher quality of RayGS. We discuss the challenges of selecting the computationally most efficient solution out of infinitely many valid ones.

The second contribution of our paper addresses MIP-related issues in a RayGS formulation. Our solution enables alias-free rendering of images at diverging test and training scales, preventing undesirable artifacts. 
For a given and normalized ray, we marginalize the Gaussian 3D distribution on a plane orthogonal to the ray and intersecting its point of maximum density. This yields a 2D Gaussian distribution which can be locally smoothed to approximate the integral over the pixel area, and to further compute the rendered opacity at each point in the pixel area. Besides the theoretically correct solution, we derive an approximated one that can be efficiently integrated into our renderer.

To summarize, our work proposes a mathematically rigorous and geometrically intuitive derivation for deriving all quantities required for high-quality and fast, hardware rasterized, ray-based Gaussian Splatting. We additionally introduce a solution for handling MIP-related issues in RayGS, demonstrated by qualitative and side-by-side comparisons. Finally, we provide quantitative evaluation results, demonstrating that we can retain state-of-the-art appearance performance of RayGS-based models while obtaining on average approximately $40\times$ faster rendering performance on scenes from the MipNeRF360~\cite{barron2022mip} and Tanks\&Temples~\cite{knapitsch2017tanks} benchmark datasets.

\section{Related Works}

3D Gaussian Splatting (3DGS) was initially presented in~\cite{kerbl3Dgaussians}, and has since become a fundamental tool in Computer Vision and Graphics.
Thanks to its speed and ease of use, 3DGS has been applied to a wide variety of downstream tasks, including text-to-3D generation~\cite{chen2023text,tang2023dreamgaussian,yi2023gaussiandreamer}, photo-realistic human avatars~\cite{zielonka2023drivable,lei2023gart,kocabas2023hugs,saito2023relightable}, dynamic scene modeling~\cite{wu20234d,luiten2023dynamic,yang2023real}, Simultaneous Localization and Mapping~\cite{matsuki2023gaussian,yugay2023gaussian,yan2023gs,keetha2023splatam}, and more~\cite{guedon2023sugar,xie2023physgaussian,ye2023gaussian}.
In this section, we focus on two areas of research that are most closely related to our work: re-formulating 3DGS as ray casting, and improving its rendering performance.

\myparagraph{3DGS as ray casting} A few different works~\cite{yu2024gaussian,hahlbohm2024efficient,huang20242d,3dgrt2024,mai2024ever,radl2024stopthepop} have shown that rendering 3D gaussian primitives using ray casting can be a preferable alternative to splatting.
These work concurrently showed that ray-splat intersections can be calculated efficiently and, importantly, \emph{exactly}, as opposed to the approximation introduced by the original rasterization formulation in~\cite{kerbl3Dgaussians}.
The works in~\cite{yu2024gaussian} and~\cite{huang20242d} steer the training process towards learning 3D representations that more accurately follow the real geometry of the scene, by introducing regularization losses that exploit the more meaningful depth and normals that can be computed with the ray-splat intersection formulation.
Similarly, the works in~\cite{radl2024stopthepop} and~\cite{hahlbohm2024efficient} exploit ray-splat intersection to compute per-pixel depth values, that can be used to locally re-order the splats and avoid ``popping'' artifacts.
While the previous works implemented ray casting in the same CUDA software rasterization framework of~\cite{kerbl3Dgaussians}, a few others~\cite{3dgrt2024,mai2024ever} exploited Nvidia hardware to implement 3DGS rendering as a full ray-tracing procedure.
While generally slower than the others, these approaches unlock an entire new range of possibilities, \eg, physically accurate simulations of reflections and shadows, by tracing light propagation through a 3DGS scene.

\myparagraph{Improving 3DGS performance}
One of the main advantages of 3DGS compared to previous photorealistic 3D reconstruction approaches such as Neural Radiance Fields (NeRF)~\cite{mildenhall2020nerf}, is its render-time performance.
Even when compared to NeRF approaches specifically tuned for speed over quality~\cite{mueller2022instant}, 3DGS can still run up to one order of magnitude faster.
Nonetheless, real-time rendering complex 3DGS scenes using the original CUDA implementation from~\cite{kerbl3Dgaussians} can be infeasible when very high output resolutions are required, if computational budget is limited, or both (\eg on VR headsets).
Because of this, optimizing 3DGS rendering performance has been an active area of research in the past years, focusing on two main directions: model pruning and compression~\cite{morgenstern2023compact,niedermayr2024compressed,niemeyer2024radsplat,fan2023lightgaussian}, and fine-tuning the rendering logic~\cite{feng2024flashgs,vkgs,fast-gauss}.

In Radsplat~\cite{niemeyer2024radsplat}, Niemeyer~\etal propose a strategy to prune splats that don't significantly contribute to image quality, considerably reducing how many need to be rendered and thus increasing speed.
Since 3DGS is generally bottlenecked by GPU memory bandwidth, model compression can be exploited to both reduce model storage size and increase rendering speed, \eg by organizing splat parameters in coherent 2D grids to be compressed using standard image compression algorithms~\cite{morgenstern2023compact}, or by developing specific parameter quantization approaches~\cite{niedermayr2024compressed,fan2023lightgaussian}. 

In FlashGS~\cite{feng2024flashgs}, Feng~\etal present an in-depth analysis of the original differentiable 3DGS CUDA renderer, proposing many small optimizations which together contribute substantial speed improvements, particularly at training time.
The largest increase in rendering performance, however, can generally be achieved by abandoning the CUDA-based software rasterization paradigm (and thus the ability to differentiate through the renderer) in favour of hardware rasterization, in a way reminiscent of older works on rendering quadratic 3D surfaces~\cite{sigg2006gpu,weyrich2007hardware}.
To the best of our knowledge, this approach to 3DGS rendering has not been formally described in computer vision literature, but many different HW-rasterization implementations of 3DGS are available as open source software, such as~\cite{vkgs,fast-gauss}.
We conclude by mentioning also the work~\cite{Gumhold2003SplattingIE}, which introduces an HW-rasterization technique for splatting illuminated ellipsoids close in spirit to the one we propose in this paper.

\section{Preliminaries: Gaussian Splatting}
We provide a brief introduction to Gaussian Splatting~(GS)~\cite{kerbl3Dgaussians} and its ray-based variant (RayGS)~\cite{yu2024gaussian}, including implementation details of hardware-rasterized  GS.

\myparagraph{Scene representation} Gaussian Splatting introduces a scene representation expressed in terms of (3D Gaussian) \emph{primitives} $\mathcal S\coloneqq\{(\vmu_i,\vE_i,o_i,\vct\xi_i)\}_{i=1}^N$, each consisting of a center $\vmu_i\in\mathbb R^3$, a covariance matrix $\vE_i\in\mathbb R^{3\times 3}$, a \emph{prior} opacity scalar $o_i\in[0,1]$ and a feature vector $\vct\xi_i\in\mathbb R^d$ (\eg RGB color). 
The covariance matrix $\vE_i$ is typically parametrized with a rotation $\mat R_i$ and a positive-definite, diagonal matrix $\mat S_i$ as follows: $\vE_i\coloneqq \mat R_i\mat S^2_i\mat R_i^\top$.
We assume center and covariance to be expressed in the camera frame. 

\myparagraph{Scene rendering} Rendering a scene $\mathcal S$ on a given camera is formulated as a per-pixel, convex linear combination of primitives' features, \ie
\begin{equation}\label{eq:rendering}
\mathcal R(\vct x;\mathcal S)\coloneqq\sum_{i=1}^N\vct \xi_{\nu_i}\omega_{\nu_i}(\vct x;\mathcal S)\prod_{j=1}^{i-1}[1-\omega_{\nu_j}(\vct x;\mathcal S)]\,,
\end{equation}
where $\nu$ is a permutation of primitives that depends on the camera pose and camera ray $\vct x$, typically yielding an ascending ordering with respect to depth of the primitive's center. The term $\omega_i(\vx;\mathcal S)\in[0,1]$ can be regarded as the \emph{rendered} primitive opacity value, which is given by
\begin{equation}\label{eq:rend_opacity}
\omega_i(\vx;\mathcal S)\coloneqq o_i\exp\left[-\frac{1}{2}\fD(\vx;\vmu_i,\vE_i)\right]\,,
\end{equation}
where $\fD(\vx;\mu_i,\vE_i)$ is a divergence of $\vx$ from the primitive. This divergence takes different forms depending on the type of model we consider, namely GS or RayGS.

\myparagraph{Support of a primitive and its boundary} Given a primitive $(\vmu, \vE, o, \vct\xi)$, the set of camera rays $\vx$ for which the rendered opacity as per~\cref{eq:rend_opacity} is above a predefined, cut-off probability value $p_\text{min}$ is called the \emph{support} of the primitive. The support is camera-specific and can be characterized in terms of $\fD$ as the set of rays satisfying
\begin{equation}\label{eq:support}
\fD(\vx;\vmu,\vE)\leq\kappa\,,
\end{equation}
where $\kappa\coloneqq -2\log\left(\frac{p_\text{min}}{o}\right)$. Indeed, the relation holds if and only if $\omega_i(\vx;\mathcal S)\geq p_\text{min}$. The set of rays for which equality holds in~\cref{eq:support} forms the \emph{boundary} of the primitive's support.
Finally, if $\kappa\leq 0$, the support of the primitive is a null set and, hence, the primitive can be skipped since it is not visible given the provided cut-off probability. For this reason, we assume $\kappa>0$ in the rest of the paper.

\subsection{Gaussian Splatting}\label{ss:gs}
Let $\pi(\vx)$ be the camera projection function mapping a 3D point in camera space to the corresponding 2D pixel in image space. In GS the divergence is given by
\begin{equation}
\fDgs(\vx;\vmu,\vE)\coloneqq (\pi(\vx)-\pi(\vmu))^\top\inv{\vE_\pi}(\pi(\vx)-\pi(\vmu))\,,
\end{equation}
where $\vE_\pi\coloneqq\mat J_\pi\vE\mat J_\pi^\top\in\mathbb R^{2\times 2}$
with $\mat J_\pi\in\mathbb R^{2\times 3}$ being the Jacobian of the projection function $\pi$ evaluated at $\vmu$. This is the Mahalanobis distance between the pixel corresponding to $\vx$ and the 2D Gaussian distribution that is obtained from the primitive's 3D Gaussian distribution transformed via a linerization of $\pi$ at $\vmu$.
One advantage of this approximation is that the support of the primitive spans a 2D ellipse in pixel-space and can be rasterized in hardware. The disadvantage is that the support is misplaced with respect to the 3D Gaussian density, yielding unexpected artifacts, and the approximation assumes a pinhole camera model.

\subsection{Ray-Based Gaussian Splatting}\label{ss:ray_gs}
RayGS improves the rendering quality by dropping the approximation in GS due to the local linearization, which causes unexpected behavior (see~\cref{fig:artifacts}).
The idea is to render a primitive by considering the point of maximum Gaussian density along each camera ray. By doing so, the density of the rendered primitive behaves as expected, but the computational cost is higher. The divergence function underlying RayGS takes the following form
\begin{equation}\label{eq:Q_ray}
\fDray(\vx;\vmu,\vE)\coloneqq(\tau(\vx)\vx-\vmu)^\top\viE(\tau(\vx)\vx-\vmu)\,,
\end{equation}
where $\tau(\vx)\coloneqq\frac{\vx^\top\viE\vmu}{\vx^\top\viE\vx}$. Geometrically, $\tau(\vx)\vx$ is the point of maximum density along ray $\vx\in\mathbb R^3\setminus\{\vct 0\}$ (see~\cref{prop:tau}), and the divergence corresponds to the Mahalanobis distance between this point and the primitive's 3D Gaussian.

\myparagraph{Skipping cases} Primitives for which $c^2\leq\kappa$ holds with 
\begin{equation}\label{eq:c}
c\coloneqq\sqrt{\vmu^\top\viE\vmu}    
\end{equation}
have a support that spans the entire image (see~\cref{prop:full_support}). This intuitively happens because the camera is positioned inside the primitive. For this reason, we skip those cases and, therefore, we assume $c^2>\kappa$ in the rest of the section.

\subsection{Hardware-Rasterized GS}\label{ss:rasterized_gs}
One advantage of GS over RayGS is the straightforward mapping of the algorithm onto a traditional hardware-accelerated rasterization pipeline with programmable vertex and fragment shading stages. The idea is simple. Since the support of a primitive yields a 2D ellipse on the image plane, it is possible to enclose it with a \emph{quad}, \ie a $4$-sided polygon, spanning a minimum area. The role of the vertex shader is to compute the positions of the quad vertices and initialize vertex-specific features that are interpolated and transformed by the fragment shader to deliver a per-pixel RGBA color. A standard alpha-blending pipeline configuration combines the RGBA color from multiple pre-sorted quads to mimic the actual rendering equation in~\cref{eq:rendering}. Below, we review the vertex and fragment shaders and refer the reader to the code of~\cite{vkgs} for more details.

\myparagraph{Vertex shader} Given the eigendecomposition, $\vE_\pi=\mat U_{gs}\mat\Lambda_{gs}\mat U_{gs}^\top$, 
the vertices of the quad enclosing a primitive's support in image-space are given by
\begin{equation}\label{eq:Vgs}
\mat V_\text{gs}\coloneqq\mat T_\text{gs}\fH(\mat Z_\text{gs})\,,
\end{equation}
where 
$\mat T_{\text{gs}}\coloneqq
\begin{bmatrix}
\mat U_{\text{gs}}\msqrt{\mat\Lambda_{\text{gs}}}
&\pi(\vmu)
\end{bmatrix}\in\mathbb R^{2\times 3}$, function $\fH$ turns each column of the argument matrix into homogeneous coordinates, 
$\mat Z_\text{gs}\coloneqq\sqrt{\kappa}\mat O\in\mathbb R^{2\times4}$ and $\mat O\coloneqq\begin{bmatrix}-1&-1&1&1\\-1&1&1&-1\end{bmatrix}$ are the vertices of the \emph{canonical quad}, namely a centered 2D square.

\myparagraph{Fragment shader} The goal of the fragment shader is to compute the rendered primitive's opacity in \cref{eq:rend_opacity} by using hardware interpolation capabilities over vertex-specific quantities. Here, the relevant part of the opacity computation is $\fDgs(\vx;\vmu,\vE)$, which changes over pixels. Given any ray $\vx$ intersecting the primitive's quad, there exist interpolating coefficients $\vct\alpha\in\mathbb R^4$ such that $\fDgs(\vx;\vmu,\vE)=\fDgs(\mathcal H(\mat V_\text{gs}\vct\alpha);\vmu, \vE)$. Moreover, for any such $\vct\alpha$
we have that
\[
\fDgs(\mathcal H(\mat V_\text{gs}\vct\alpha);\vmu, \vE)=\Vert\mat Z_\text{gs}\vct\alpha\Vert^2\,.
\]
Accordingly, it is sufficient to interpolate $\mat Z_\text{gs}$ and use a simple dot product in the fragment shader to compute $\fDgs(\vx;\vmu,\vE)$ over the quad area. The resulting quantity is then fed to a pixel-independent function to get the rendered opacity as per~\cref{eq:rend_opacity}.

\section{Hardware-Rasterized RayGS}
In this section, we introduce the main contribution of the paper, namely showing how hardware rasterization can be used to efficiently render Gaussian primitives under a ray-based formulation. This allows to retain the advantage of GS, \ie faster rendering, and the better quality of RayGS, due to the dropped linear approximation.


Thanks to the linear approximation of the projection function, the support of a primitive in GS spans a 2D ellipse on the image plane and an optimal enclosing quad can be efficiently computed in the same space, as shown in~\cref{ss:rasterized_gs}. 
However, when it comes to RayGS, where we drop this approximation, the support of primitives can also span half-hyperbolas making the approximation with quads \emph{on the image plane} more complex and less efficient. 

To sidestep this limitation, we drop the restriction of seeking quads on the image plane and approximate the support of primitives with quads placed directly in 3D space. In fact, any quad intersecting all camera rays in the support of a primitive is a valid solution and there are infinitely-many ones. But, two valid quads do not necessarily share the same computational efficiency, so making the right choice here makes the difference.

In this section, we present a solution strategy using quads lifted in 3D space that can be computed efficiently as we will show later in the experimental section.
Mimicking \cref{ss:rasterized_gs}, we discuss how the vertex and fragment shaders are implemented in our solution, which are the distinctive parts of our contribution, while we omit the rest of the pipeline (\eg alpha-blending logic), for it is shared with the GS hardware-rasterized implementation~\cite{vkgs}.


\subsection{Vertex shader}
Consider a primitive with center $\vmu$ and covariance $\vE\coloneqq \mat R_p\mat S^2\mat R_p^\top$ that factorizes in terms of the scale matrix $\mat S$ and rotation matrix $\mat R_p$.
For a given camera view, we compute a quad that encloses the 3D points of maximum Gaussian density that we find along rays belonging to the boundary of the primitive's support, \ie the points belonging to the following set (see, \cref{fig:isomorphism}a):
\[
\mathcal E\coloneqq\left\{\tau(\vx)\vx\,:\, \fDray(\vx;\vmu,\vE)=\kappa,\,\vct x\in\mathbb R^3\setminus\{\vct 0\}\right\}\,.
\]
Determining such a quad is possible because $\mathcal E$ forms a 2D ellipse embedded in 3D space and, hence, is isomorphic to the unit circle $\mathbb S_1$. 
To grasp the geometrical intuition of why this is the case, we describe how we can map $\mathcal E$ to $\mathbb S_1$ and 
provide in \cref{fig:isomorphism} a schematic overview.
\begin{figure}[thb]
    \centering
    \includegraphics[width=.8\columnwidth]{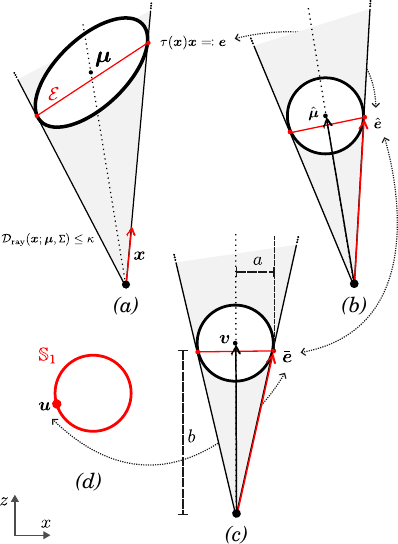}
    \caption{Schematic overview of the isomorphism between $\mathcal E$ and the unit circle $\mathbb S_1$, shown from the $(x,z)$-plane perspective.} 
    \label{fig:isomorphism}
    \vspace{-10pt}
\end{figure}

\myparagraph{Isomorphism $\Phi$ between $\mathcal E$ and $\mathbb S_1$}
We start with the primitive in its original space in \cref{fig:isomorphism}a.
By~\cref{prop:cond} we have that 
\begin{equation}\label{eq:rel1}
\vct e^\top\viE \vct e=\vct e^\top\viE\vct\mu=c^2-\kappa\,,
    \end{equation}
holds for all $\vct e\in\mathcal E$ with $c$ as defined in~\cref{eq:c}.
By setting $\vmuh\coloneqq\frac{1}{c}\inv{\mat S}\mat R_p^\top\vmu$ and $\hat{\vct e}\coloneqq\frac{1}{\sqrt{c^2-\kappa}}\inv{\mat S}\mat R_p^\top\vct e$, we can rewrite the right-most equality in~\cref{eq:rel1} as
\begin{equation}\label{eq:rel2}
\hat{\vct e}^\top\vmuh=b\coloneqq\sqrt{1-\frac{\kappa}{c^2}}\,.    
\end{equation}
The result of this space transformation is shown in \cref{fig:isomorphism}b. Since both $\hat{\vct e}$ and $\vmuh$ are points of the unit sphere $\mathbb S_2$, we have that all $\hat{\vct e}$ satisfying~\cref{eq:rel2} live on a circle embedded in 3D space. 
To make this mapping explicit, we rotate the space to align $\vmuh$ with the $z$-axis $\vct v=\begin{bmatrix}0&0&1\end{bmatrix}^\top$. To this end, we define a rotation matrix
$\vRv$ (see,~\cref{ss:rot}) such that $\hat\vmu=\vRv\vct v$.
By setting $\bar{\vct e}\coloneqq \vRv\hat{\vct e}$, we have that
\begin{equation}\label{eq:rel3}
\bar{\vct e}^\top\vv=b\,.        
\end{equation}
This transformation is depicted in \cref{fig:isomorphism}c. Since $\bar{\vct e}$ is still a vector of the unit sphere $\mathbb S_2$, and by~\cref{eq:rel3} its $z$-coordinate has to be $b$, we have that
\begin{equation}\label{eq:rel4}
\bar{\vct e}=b\fH\left(a\vct u\right)
\end{equation}
holds with $a\coloneqq\sqrt\frac{\kappa}{c^2-\kappa}$ and for a specific element $\vct u\in\mathbb S_1$ of the unit circle, which corresponds to the $(x,y)$-subvector of $\bar{\vct e}$ normalized to unit length (see~\cref{fig:isomorphism}d).
In summary, we have described a transformation chain 
\[
\vct e\in\mathcal E\bij\hat{\vct e}\in\mathbb S_2\bij\bar{\vct e}\in\mathbb S_2\bij\vct u\in\mathbb S_1
\]
mapping elements of $\mathcal E$ to the unit circle, which is both linear and invertible, thus showing that the two sets are isomorphic. From this we infer that all points in $\mathcal E$ actually live on a 2D ellipse embedded in 3D space and, therefore, can be enclosed with a quad placed on the same 3D plane the 2D ellipse belongs to. We denote by $\Phi$ the isomorphism from $\mathcal E$ to $\mathbb S_1$. The 3D center of the ellipse spanned by $\mathcal E$ can be computed as $\Phi^{-1}(\vct 0)$, \ie by back-mapping the center of the unit disk $\vct 0\in\mathbb R^2$.

\begin{figure}
    \centering
    \includegraphics[width=0.8\columnwidth]{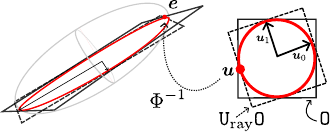}
    \caption{Examples of 3D quads obtained by mapping 2D squares via the isomorphism $\Phi^{-1}$.}
    \label{fig:quad}
    \vspace{-15pt}
\end{figure}
\myparagraph{Quad vertices and optimality}
Determining the vertices of a quad enclosing $\mathcal E$ becomes easy given the mapping $\Phi$, because it is sufficient to enclose the unit circle $\mathbb S_1$ with a 2D quad and map it back to 3D space with $\Phi^{-1}$ (see,~\cref{fig:quad}a). The explicit form of $\Phi^{-1}$ can be obtained by traversing backwards the transformations from the previous paragraph, with some terms rearranged:
\begin{equation}\label{eq:invPhi}
\Phi^{-1}(\vct u)\coloneqq \frac{c^2-\kappa}{c}\underbrace{\mat R_p\mat S\vRv}_{\eqqcolon\mat Q}\mat\fH(a\vct u)\,.    
\end{equation}
Unfortunately, there are infinitely-many ways we can approximate the unit circle with a quad, two examples being given in \cref{fig:quad}, but not all are efficient for the sake of rendering. In particular, their projection on the image plane can span different area sizes, potentially introducing a waste of compute on irrelevant pixels. Ideally, we would like to position the quad in a way to minimize the spanned area \emph{on the image plane}, but this requires additional complexity in the way the vertices are computed, to the detriment of the overall rendering speed. What we found out to be a good compromise is to position the vertices of the enclosing quad in a way to span the smallest area \emph{in 3D space}, by forming a tight rectangle aligned with the 2D ellipse's axes (see~\cref{fig:quad}b).
To this end, we identify one vertex of the ellipse $\mathcal E$, \ie one of the two endpoints along its major axis, by localizing it first on the unit circle. This is achieved by solving the following optimization problem
\begin{equation}\label{eq:optim}
\vct u_1\in\arg\max_{\vct u\in\mathbb S_1}\Vert\Phi^{-1}(\vct u)-\Phi^{-1}(\vct 0)\Vert^2\,,
\end{equation}
which finds the point $\vct u_1$ on the unit circle whose counterpart $\Phi^{-1}(\vct u)$ on the ellipse $\mathcal E$ maximizes the distance to the ellipse's center, which can computed as $\Phi^{-1}(\vct 0)$. 
By substituting~\cref{eq:invPhi} into~\cref{eq:optim}, and dropping constant scalar factors that do not change the maximizers, the objective takes a standard quadratic form (see,~\cref{ss:optim2}):
\begin{equation}\label{eq:optim2}
\vct u_1\in\arg\max_{\vct u\in\mathbb S_1}\vct u^\top\mat B\vct u\,.
\end{equation}
Here, $\mat B\coloneqq \mat Q_{0:2}^\top\mat Q_{0:2}$, where $\mat Q_{0:2}\in\mathbb R^{3\times 2}$ denotes $\mat Q$ restricted to the first two columns. The solution $\vct u_1$ is an eigenvector of $\mat B$ with maximum eigenvalue, which can be computed in closed-form for $2\times 2$ matrices. Similarly, $\vct u_0$, \ie the point corresponding to the minor axis, is an eigenvector of $\mat B$ with minimum eigenvalue. Since $B$ has only two eigenvectors and are mutually orthogonal, $\vct u_0$ can be computed by simply rotating $\vct u_1$ by 90 degrees anticlockwise. 
By doing so we ensure to preserve a consistent orientation of the quad surface.
We stack $\vct u_0$ and $\vct u_1$ to form a $2\times 2$ matrix $\mat U_\text{ray}\coloneqq (\vct u_0,\vct u_1)$, which we use to rotate the vertices of the canonical quad $\mat O$ before mapping it back to 3D space via $\Phi^{-1}$ (see,~\cref{fig:quad}). Although this mapping would already provide a valid 3D quad for rendering, we also scale it by the factor $\frac{c^2}{c^2-\kappa}$. This scaling operation preserves the support of camera rays, making the scaled quad equivalent to the original one from a rendering perspective, but it enables the computation of vertices in a more efficient and stable way.
In fact, the final set of quad vertices $\mat V_\text{ray}$ can be computed as follows  (see~\cref{ss:V_ray}), mimicking~\cref{eq:Vgs}:
\begin{equation}\label{eq:V_ray}
\mat V_\text{ray}\coloneqq \mat T_\text{ray}\mathcal H(\mat Z_\text{ray})\,,
\end{equation}
where $\mat Z_\text{ray}\coloneqq \frac{\sqrt \kappa}{b}\mat O$ and $
\mat T_\text{ray}\coloneqq \begin{bmatrix}\mat Q_{0:2}\mat U_\text{ray}&\vmu\end{bmatrix}$.
\myparagraph{A note on near plane clipping}
The proposed formulation works under the assumption that near plane clipping is disabled. Indeed, if a quad intersecting the near plane is clipped, we obtain undesired effects like visible discontinuities (see~\cref{fig:cipping}). Nonetheless, frustum culling of primitives based on a near plane is still applicable without consequences, for the whole quad is removed in that case.
\begin{figure}[htb]
    \centering
    \includegraphics[width=0.6\linewidth]{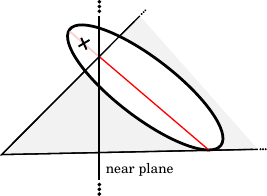}
    \caption{Undesired effects of near-plane clipping. Sharp discontinuities might be visible if a quad intersects the clipping plane.}
    \label{fig:cipping}
    \vspace{-10pt}
\end{figure}

\subsection{Fragment shader}
Our goal is to compute the rendered primitive's opacity in~\cref{eq:rend_opacity} by exploiting hardware interpolation of quantities specified at the quad's vertices. Again, the relevant part of the computation is $\fDray(\vx;\vmu,\vE)$ since it changes over pixels.
Similarly to GS, we have that for any ray $\vx$ intersecting the primitive's quad, there exist interpolating coefficients $\vct\alpha\in\mathbb R^4$ such that $\fDray(\vx;\vmu,\vE)=\fDray(\mat V_\text{ray}\vct\alpha;\vmu,\vE)$. Moreover, for any such $\vct\alpha$ the following holds (see,~\cref{ss:fragment}):
\begin{equation}
\fDray(\mat V_\text{ray}\vct\alpha;\vmu,\vE)=\inv{\left\{c^{-2}+\Vert\mat Z_\text{ray}\vct\alpha\Vert^{-2}\right\}}\,.
\end{equation}
Hence, it is sufficient to interpolate $\mat Z_\text{ray}$ and apply a simple scalar function to get $\fDgs(\vx;\vmu,\vE)$ over the quad area. The resulting quantity is then fed to a pixel-independent function to get the rendered opacity as per~\cref{eq:rend_opacity}.

\section{MIP for RayGS}\label{sec:mip}
In this section we focus on MIP-related issues, which should be addressed to ensure higher-quality renderings in particular for VR applications, where we are free of moving in the scene.
The problem arises from the fact that when we render we approximate a pixel with a single camera ray in the center, instead of considering the whole pixel area. Imagine to have a primitive that spans less than a pixel when rendered. If the center of the pixel overlaps with the support of the primitive, the full pixel will take the primitive's rendered color. Otherwise it will show the background color. 
We would instead expect the pixel color to be the combination of primitive and background color, depending on the size of the primitive's support.

In the context of GS, there have been works suggesting ways to address MIP issues~\cite{yu2024mip}, but to our knowledge no work has explicitly addressed the matter for RayGS. To fill this gap, we introduce a novel formulation that can be easily integrated into our fast renderer.
Given a (normalized) ray $\vx$, the idea is to marginalize the Gaussian 3D distribution on a plane that is orthogonal to $\vx$ and passing through the point of maximum density $\tau(\vx)\vx$. This yields a 2D Gaussian distribution on the same plane, which can be smoothed using a properly-sized isotropic 2D Gaussian distribution to approximate the integral over the pixel area. By applying this idea, we end up with the following per-primitive distribution over (normalized) camera rays (see,~\cref{ss:P} for a detailed derivation):
\begin{equation*}\label{eq:P}
P_\text{MIP}(\vx)\propto\frac{\tau^2(\vx)}{\sqrt{|\hat{\mat\Sigma}_x|\vx^\top\inv{\hat\vE_x}\vx}}\exp\left(-\frac{1}{2}\fDray(\vx;\vmu,\hat\vE_x)\right)\,,    
\end{equation*}
where $\hat\vE_x\coloneqq\vE+\sigma_x^2\tau(\vx)^2\mat I$ is a pixel-dependent 3D covariance matrix and $\sigma^2_x$ represents the area of the pixel at unit distance along the ray, which is in principle also dependent on the ray because the plane we are projecting onto is not necessarily parallel to the image plane. 
Akin to MIP-Splatting~\cite{yu2024mip}, we can compute the rendered opacity by considering the exponential term in $P_\text{mip}(\vct x)$ and modulating the prior opacity so that the total opacity matches the one of standard RayGS, where the modulating factor is given by $\sqrt{\frac{|\mat\Sigma|c^2}{|\hat{\mat\Sigma}_x|\vx^\top\inv{\hat\vE_x}\vx}}$.
The resulting formula is structurally similar to the one from~\cite{yu2024mip}, differences being that we do not require separate 2D and 3D filters and the modulation factor is pixel-dependent. Unfortunately, the latter fact, despite being theoretically more correct, poses challenges when it comes to having a fast rasterizer integrating it. For this reason, we introduce the following approximations in our implementation. We start considering $\sigma_x$ constant and regard $\hat\vE$ as the resulting pixel-invariant counterpart of $\hat\vE_x$. Next, we assume $\vx^\top\inv{\hat\vE}\vx\approx\hat c^2\coloneqq \vmu^\top\inv{\hat\vE}\vmu$. This approximation is accurate when primitives are sufficiently small to have rays concentrated around the mean, which accounts for most of the cases. Conversely, when primitives are sufficiently large along both $(x,y)$-directions then the modulation factor is close to $1$ for all $\vct x$, so the approximation still works. The cases when the approximation is less accurate are more rare, \ie when primitives are thin less than a pixel upon projection, but sufficiently elongated to span a wide cone of rays. 
The final form of our MIP formulation is thus
\begin{equation}\label{eq:mip}
\omega^\text{mip}(\vx)\approx o_i\sqrt{\frac{|\mat\Sigma|c^2}{|\hat{\mat\Sigma}|\hat c^2}}\exp\left(-\frac{1}{2}\fDray(\vx;\vmu,\hat\vE)\right)\,,    
\end{equation}
which can be efficiently integrated in our fast renderer.

\section{Experiments}
We have implemented our hardware-rasterized renderer for RayGS  on top of VKGS~\cite{vkgs}, which is based on Vulkan. Accordingly, we refer our method to as \emph{VKRayGS}. However, other OpenGL implementations are potentially possible based on what we described in the paper. 
Our experimental evaluation targets different goals:
    
\myparagraph{Rendering speed} We want to show that our contribution unlocks significantly higher rendering speed for RayGS models than the best publicly-available, open-sourced alternative, which at the time of writing is the CUDA-based renderer from Gaussian Opacity Field (GOF)~\cite{yu2024gaussian}. Moreover, to put the numbers in the right perspective, we provide in~\cref{ss:GSvsVKGS} a speed comparison between the original VKGS implementation of GS and the CUDA-based implementation of Gaussian Splatting (GS) from INRIA~\cite{gs-inria}, in its most recent version integrating speed optimizations.
    
\myparagraph{Rendering quality} The way primitives are rendered in our formulation is mathematically equivalent to~\cite{yu2024gaussian}, but there are still factors that can influence the final rendering quality, like differences occurring in the rest of the pipeline (\eg how primitives are clipped). For this reason, we report quality metrics in addition to speed measurements.  However, given that we use models trained with the differentiable renderer from the competitors, it is reasonable to expect a bias in their favor. 
Since quality drops that are not directly ascribable to our contribution are expected to show up also when comparing VKGS against its CUDA counterpart, we decided to include in~\cref{ss:GSvsVKGS} also the latter comparison despite not being directly related to our contribution.

\subsection{Evaluation Protocol}
We evaluate the performance of our renderer on scenes from two benchmark datasets that have been often used in the context of novel-view synthesis with GS (see \eg~\cite{kerbl3Dgaussians,rota2024revising}), namely MipNerf360~\cite{barron2022mip} and Tanks and Temples~\cite{knapitsch2017tanks}. Those are real-world captures that span both indoor and outdoor scenarios.
For each scene, every $8$th image is set aside to form a test set, on which three quality metrics are reported, namely peak signal-to-noise ratio (PSNR),
structural similarity (SSIM) and the perceptual metric from (LPIPS) using VGG. Following the standard protocol from~\cite{barron2022mip}, we evaluate MipNerf360 indoor/outdoor scenes using the $2\times$/$4\times$ downsampling factors, respectively. Moreover, we use downsampled images provided by the authors of the dataset. For Tanks\&Temples we evaluate the results at $2\times$ downsampling factor, akin to previous methods.
The same images used for quality evaluations are also used to compute the rendering speed in terms of Frames-Per-Second (FPS). We opted to run experiments on an RTX2080, which is a mid-range GPU, sufficiently powerful and CUDA-based to run the implementations from the competitors. Nevertheless, our implementation can run on any GPU supporting Vulkan and, more in general, enables implementations in OpenGL, which broadens up the applicability spectrum significantly. All scene models used for the experiments are either provided by the authors of the competing methods or, if not available, have been trained using the code they provide.
Finally, the quantitative evaluations of our model are run with MIP disabled, because the scene models from~\cite{yu2024gaussian} already incorporate an additive factor on the 3D covariance and are trained without the MIP opacity modulation. We nevertheless provide qualitative examples showing the importance of the MIP formulation.

\subsection{Results}
Before delving into the results obtained by our fast renderer against GOF, and following our previous discussion about quality expectations, we discuss a quantitative comparison of VKGS against GS. In addition, we provide some qualitative results and discussion about MIP. 
\begin{figure}[tb]
    \centering
    \includegraphics[width=.48\textwidth]{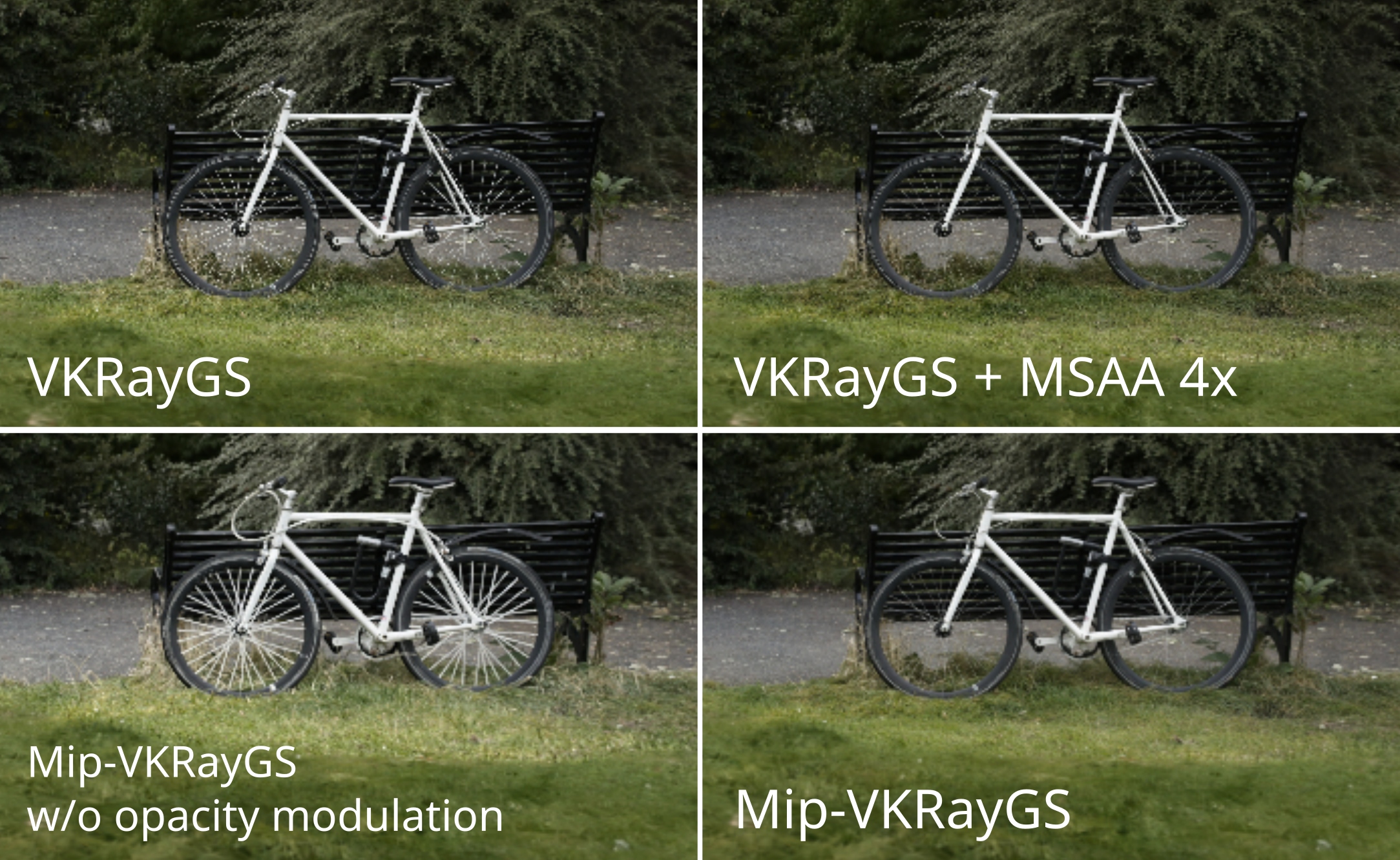}
    \caption{Benefits of our MIP formulation for RayGS. Best viewed with digital zoom. See text for details.}
    \label{fig:mip}
    \vspace{-10pt}
\end{figure}

\myparagraph{GS versus VKGS}
In~\cref{ss:GSvsVKGS}, we report results obtained by GS versus the Vulkan counterpart VKGS on the benchmark datasets, which highlight a clear speed advantage of the Vulkan implementation over the CUDA-based one from GS, being on average $2\times$ faster. 
Quality metrics, instead, are not significantly different, excepting a few cases. However, the fact that there are differences
indicates potential misalignment between the implementations and the results favor GS because the model has been trained with the same renderer.
\begin{figure*}[thb]
    \centering
    \includegraphics[width=.9\textwidth]{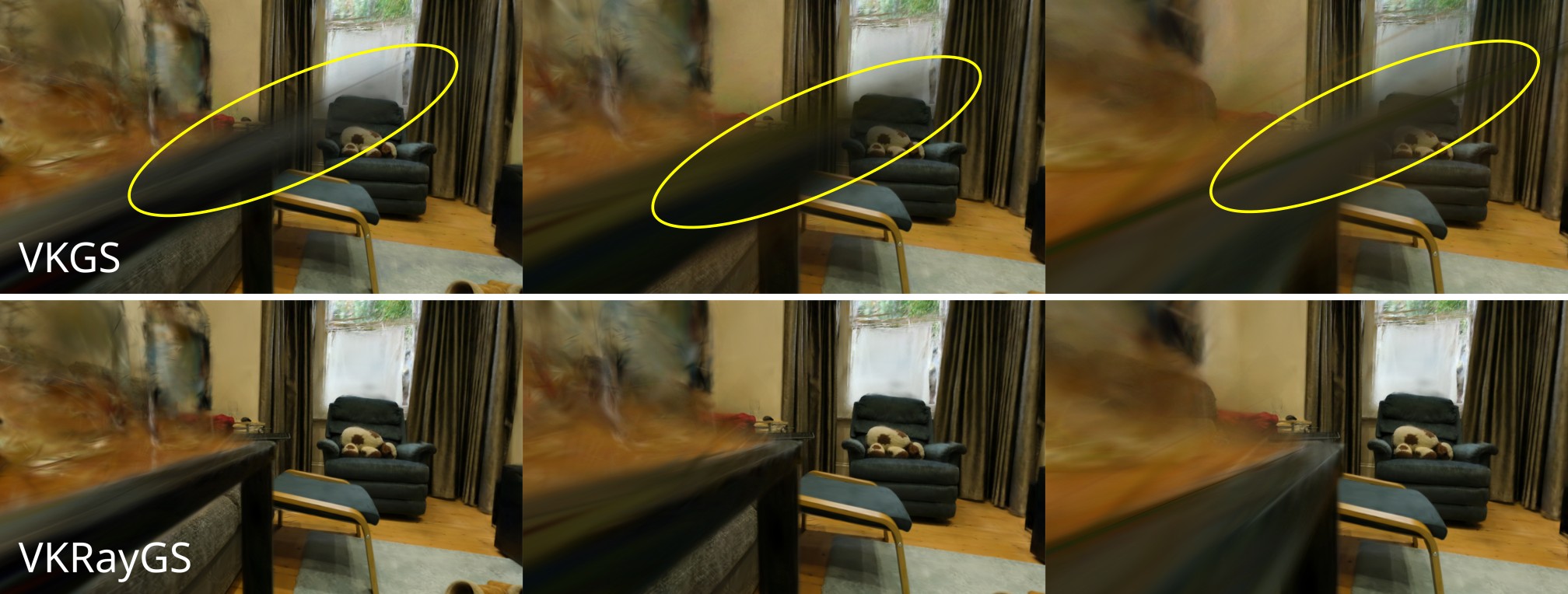}
    \caption{Artifacts one might experience with GS models (top) as opposed to RayGS (bottom), while moving close to objects. Renderings are from the \emph{room} scene of MipNeRF360 and obtained with VKGS and our VKRayGS, respectively. }
    \label{fig:artifacts}
    \vspace{-10pt}
\end{figure*}

\myparagraph{GOF versus VKRayGS}
In~\cref{tab:quantitative_raygs}, we report the results obtained by our renderer against GOF. We have split the table into two sections to distinguish scenes from MipNerf360 (top) and Tanks\&Temples (bottom). For each scene, we report the size of the model in terms of number of primitives $N$ and report left-to-right speed comparisons in terms of FPS and quality metrics in terms of PSNR, SSIM and LPIPS, averaged over all the test views.
we observe huge gains in terms of speed, with an average speed up of $40\times$, highlighting the relevance of our contribution, for it unlocks real-time rendering with the better RayGS models (as reflected by the overall worse perceptual metrics obtained with GS in~\cref{tab:quantitative_gs}). If we shift our focus on the qualitative metrics, we observe a slight quality drop with some cases where our renderer delivers even better perceptual scores. This drop is motivated by implementation misalignments between GOF and VKRayGS that are not ascribable to our method, as confirmed by similar, if not larger, gaps between VKGS and GS (see, ~\cref{tab:quantitative_gs}).
Besides the quantitative analysis, qualitative comparisons are provided in~\cref{sec:qualitative}, where perceptual differences are barely visible.
\begin{table}[thb]
    \centering
    \setlength{\tabcolsep}{2pt}
    \resizebox{\columnwidth}{!}{
    \begin{tabular}{c||c||cc||cc|cc|cc}\toprule
        {\footnotesize on RTX2080} &&\multicolumn{2}{c}{FPS $\uparrow$ }&\multicolumn{2}{c}{PSNR $\uparrow$}&\multicolumn{2}{c}{SSIM $\uparrow$}&\multicolumn{2}{c}{LPIPS $\downarrow$}\\
 Scene&$N$&GOF&VKRayGS&GOF&VKRayGS&GOF&VKRayGS&GOF&VKRayGS\\\midrule\midrule
 bicycle&5.35M&4&\bf 177&25.47&25.31&0.784&0.784&0.206&0.203\\
bonsai&1.07M&6&\bf 341&31.60&31.46&0.937&0.930&0.240&0.204\\
counter&0.82M&6&\bf 292&28.69&28.57&0.901&0.896&0.258&0.230\\
flowers&3.28M&6&\bf 191&21.67&21.61&0.632&0.631&0.309&0.310\\
garden&4.40M&4&\bf 172&27.46&27.34&0.866&0.865&0.122&0.117\\
kitchen&1.07M&5&\bf 242&30.74&30.63&0.915&0.911&0.168&0.153\\
room&1.09M&5&\bf 305&30.81&30.55&0.915&0.906&0.281&0.245\\
stump&4.97M&6&\bf 186&26.96&26.94&0.790&0.792&0.223&0.22\\
treehill&4.02M&5&\bf 184&22.40&22.48&0.638&0.636&0.325&0.326\\
\midrule
barn&0.83M&10&\bf 312&28.99&28.30&0.892&0.883&0.190&0.192\\
caterpillar&1.43M&7&\bf 247&23.68&23.12&0.820&0.810&0.241&0.246\\
ignatius&2.60M&7&\bf 203&22.76&22.42&0.826&0.815&0.186&0.188\\
meetingroom&0.95M&7&\bf 351&25.50&24.71&0.881&0.867&0.235&0.230\\
truck&2.14M&7&\bf 195&25.80&25.30&0.892&0.881&0.153&0.141\\
\bottomrule
    \end{tabular}}
    \caption{Comparison between GOF and VKRayGS on scenes from MipNeRF360 and Tanks\&Temples. }
    \label{tab:quantitative_raygs}
    \vspace{-10pt}
\end{table}


\myparagraph{Benefits of MIP}
To understand why it is important to address MIP-related issues, in particular in real-time viewers, we provide in~\cref{fig:mip} an example from the MipNeRF360 \emph{bicycle} scene. We show crops of the bike from images rendered by a far away camera. 
Top-left, we provide the rendering with plain VKRayGS, which exhibits strong aliasing artifacts. Top-right, we adopt a $4\times$ multi-sampling anti-aliasing (MSAA) approach, which consists in averaging $4$ rays per pixel, but also this solution solves  only partially the issue at a higher computational cost. Bottom-left, we show renderings from VKRayGS with our MIP formulation, when we set $\sigma^2\coloneqq 0.1$ but neglect the opacity modulation factor. This solution solves the aliasing problem, but renders primitives unnaturally thick (see \eg the bike wheel rays). Finally, bottom-right, we have VKRayGS with the full MIP formulation, which produces an antialiased output without artifacts at a negligible computational overhead.

\myparagraph{GS versus RayGS}
In~\cref{fig:artifacts} we highlight some artifacts that typically occur when using a GS renderer like VKGS as opposed to a RayGS one like ours. We show three frames of a camera that moves along a linear trajectory in the \emph{room} scene of MipNeRF360. The camera moves close to objects in the scene on purpose, as this is typically the setting under which artifacts occur. On the top row, we present the results with VKGS, which exhibit spikes inconsistent with the scene geometry. On the bottom row, we report results obtained with VKRayGS, which despite being RayGS-based is executed on a GS scene model. As we can see, the artifacts afflicting the GS renderer are not there. This is because in RayGS models the rendered opacity is geometrically more consistent as opposed to GS.



\section{Conclusions}
We have presented a novel approach to rendering ray-based 3D Gaussian Splatting (RayGS) using hardware rasterization, achieving both fast and high-quality results for novel view synthesis. Our method leverages the advantages of RayGS, which provides superior quality compared to traditional 3DGS, while obtaining significant rendering speed gains on all tested scenes. We have demonstrated that our approach can render high-quality images at frame rates suitable for VR and MR applications.
Our contributions include a mathematically rigorous and geometrically intuitive derivation of how to efficiently estimate all relevant quantities for rendering of RayGS models, as well as a solution to MIP-related issues in a RayGS formulation, enabling alias-free rendering of scenes at diverging test and training scales. 

Our work has shown how to significantly speedup RayGS models at test time, but it would be interesting to employ hardware rasterization to improve training time as well. It is not trivial how this can be achieved and we leave this to future work.

{
    \small
    \bibliographystyle{ieeenat_fullname}

}
\clearpage
\appendix
\section{Derivations}
\begin{proposition}\label{prop:tau}
    $\tau(\vct x)\vct x$ is the point of maxiumum density along ray $\vct x\in\mathbb R^3\setminus\{\vct 0\}$ for a Gaussian distribution with parameters $(\vmu, \vE)$.
\end{proposition}
\begin{proof}
Let $\alpha\vct x$ parametrize all points along ray $\vct x$. The point of maximum density for the given 3D Gaussian is obtained by solving
\[
\min_\alpha~\quad f(\alpha)\coloneqq (\alpha\vct x-\vmu)^\top \viE(\alpha\vct x-\vmu)\,.
\]
Since the problem is convex, we find the solution by simply setting to zero the derivative of $f$, yielding
\[
\alpha^\star\coloneqq \tau(\vct x)\,,
\]
where $\tau$ is defined as per \cref{ss:ray_gs}. It follows that $\alpha^\star\vct x = \tau(\vct x)\vct x$ is the point of maximum density along the ray as required.
\end{proof}
\begin{proposition}\label{prop:full_support}
If $c^2\leq\kappa$ holds for a given Gaussian primitive, then its support spans the whole image.
\end{proposition}
\begin{proof}
By plugging~\cref{eq:Q_ray} into the support condition~\cref{eq:support}, unfolding the definition of $\tau(\vx)$, and after simple algebraic manipulations, we end up with the following relation:
\begin{equation}\label{eq:rel5}
\frac{(\vmu^\top\viE\vx)^2}{\vx^\top\viE\vx}\geq 0\geq c^2-\kappa\,,
\end{equation}
which holds for any $\vct x\in\mathbb R^3\setminus\{\vct 0\}$, or in other terms all possible camera rays.
\end{proof}
\begin{proposition}\label{prop:tau_nonzero}
If $c^2>\kappa$ holds for a given Gaussian primitive, then $\tau(\vx)\neq 0$ for all $\vx\in\mathbb R^3\setminus\{\vct 0\}$.
\end{proposition}
\begin{proof}
Following a derivation similar to~\cref{prop:full_support}, we arrive at the following relation
\[
\tau(\vx)(\vmu^\top\viE\vx)\geq c^2-\kappa>0\,.
\]
Since the condition $c^2>\kappa$ implies that $\vmu\neq\vct 0$, and also $\vx\neq\vct 0$, it follows that $\vmu^\top\viE\vx\neq 0$.
Hence, necessarily $\tau(\vx)\neq 0$ for all $\vx\in\mathbb R^3\setminus\{\vct 0\}$.
\end{proof}

\begin{proposition}\label{prop:tau_e}
    $\tau(\vct e)=1$ for all $\vct e\in\mathcal E$ assuming $c^2>\kappa$.
\end{proposition}
\begin{proof}
Given $\vct e\in\mathcal E$ we have by definition of $\mathcal E$ that $\vct e = \tau(\vct x)\vct x$ holds for some $\vct x\in\mathbb R^3\setminus\{\vct 0\}$. By~\cref{prop:tau_nonzero}, $\tau(\vx)\neq 0$ assuming $c^2>\kappa$, which implies that $\vct e\neq\vct 0$.
Then, we can write $\vct x = \alpha\vct e$ by setting $\alpha = \tau(\vct x)^{-1}$ and, therefore, $\vct e = \tau(\vct x) \vct x  = \tau(\alpha\vct e) \alpha\vct e = \tau(\vct e)\vct e$,  the last equality following by unfolding the definition of $\tau$. Accordingly, $\tau(\vct e)=1$ necessarily holds.
\end{proof}

\begin{proposition}\label{prop:cond}
    For any $\vct e\in\mathcal E$, the following holds:
    \[
    \vct e^\top\viE \vct e=\vct e^\top\viE\vct\mu=c^2-\kappa\
    \]
    provided that $c^2>\kappa$.
\end{proposition}
\begin{proof}
By \cref{prop:tau_e}, we have that $\tau(\vct e)=1$ holds for any $\vct e\in\mathcal E$. By definition of $\tau$ it follows that $\vct e^\top\viE \vct e=\vct e^\top\viE\vct\mu$. This together with the constraint $\fDray(\vct e;\vmu,\vE)=(\vct e-\vct\mu)^\top\viE(\vct e-\vct\mu)=\kappa$ from $\mathcal E$ yields the required relation by simple algebraic manipulations.
\end{proof}

\subsection{Derivation of $\vRv$}\label{ss:rot}
We define the rotation matrix aligning $\vct v$ to $\vmuh$ as follows
\[
\vRv \coloneqq 
\begin{cases}
\mat R_{\vmuh,\vv}&\text{if $\vmuh^\top\vct v\geq 0$}\\
\mat R_{\vmuh,-\vv}\mat P&\text{else}
\end{cases}
\]
where $\mat P\coloneqq\begin{bmatrix}-1&&\\&1&\\&&-1\end{bmatrix}$.
Here, we use the formula 
$
\mat R_{\vct x,\vct y}\coloneqq 2\frac{(\vct x+\vct y)(\vct x+\vct y)^\top}{(\vct x+\vct y)^\top(\vct x+\vct y)}-\mat I\,,
$ 
which yields a rotation matrix aligning 3D vector $\vct y$ to $\vct x$.
This is a special case of the Rodriguez rotation formula that we obtain considering a $180^\circ$ rotation around the axis $\vct x+\vct y$. Since our goal is to align $\vct v$ to $\vmuh$, $\mat R_{\vmuh,\vv }$ could already serve that purpose. Unfortunately, $\mat R_{\vmuh,\vv }$ is not defined when $\vmuh=-\vv$. For this reason, we distinguish two cases: If $\hat\vmu^\top\vct v\geq 0$, then Rodriguez rotation formula yields a valid solution, so we return $\mat R_{\vmuh,\vv }$. Otherwise, we first rotate the space $180^\circ$ around $\begin{bmatrix}
    0&1&0
\end{bmatrix}$ with matrix $\mat P$, so that $\vct v$ points in the opposite direction, and then apply $\mat R_{\vmuh,-\vv }$ to align $-\vv$ to $\vmuh$. Indeed, $\mat R_{\vmuh,-\vv}\mat P\vv=\mat R_{\vmuh,-\vv}(-\vv)=\vmuh$.

\subsection{Derivation of \cref{eq:optim2}}\label{ss:optim2}
We start simplifying the objective in~\cref{eq:optim}:
\begin{align*}
\Vert\Phi^{-1}(\vct u)&-\Phi^{-1}(\vct 0)\Vert^2\\
&\overset{(a)}{\propto}\Vert\mat Q[\fH(a\vct u)-\fH(a\vct 0)]\Vert^2\\
&\overset{(b)}{\propto}\Vert\mat Q_{0:2}\vct u\Vert^2\\
&\overset{(c)}{=}\vct u^\top\mat Q_{:,0:2}^\top\mat Q_{:,0:2}\vct u=\vct u^\top\mat B\vct u\,.
\end{align*}
Here, in (a) we unfold the definition of $\Phi^{-1}$ as per main paper, neglect multiplicative factors not depending on $\vct u$ and refactor terms; (b) follows from the fact that 
$\fH(\vct z)-\fH(\vct z')$ yields a 3D vector with a null $z$-coordinate for any $\{\vct z,\vct z'\} \subset \mathbb R^2$, so $\mat Q(\fH(\vct z)-\fH(\vct z'))=\mat Q_{:,0:2}(\vct z-\vct z')$. 
In addition, we neglected again multiplicative factors not depending on $\vct u$. Finally, (c) follows from unfolding the norm and noting that the resulting matrix of quadratic coefficients  coincides with $\mat B$ as per main paper.

\subsection{Derivation of \cref{eq:V_ray}}\label{ss:V_ray}
According to the description preceding \cref{eq:V_ray}, we have that $\mat V_\text{ray}\coloneqq \frac{c^2}{c^2-\kappa}\Phi^{-1}(\mat U_\text{ray}\mat O)$, where we assume that $\Phi^{-1}$ is applied column-wise to the input matrix. This can then be rewritten as follow:
\begin{align*}
    \mat V_\text{ray}&\overset{(a)}{=}c\mat Q\mathcal H(a \mat U_\text{ray}\mat O)\\
    &\overset{(b)}{=}c\mat Q\mathcal H\left(\frac{1}{c} \mat U_\text{ray}\mat Z_\text{ray}\right)\\
    &\overset{(c)}{=} c\left[\frac{1}{c} \mat Q_{0:2} \mat U_\text{ray}\mat Z_\text{ray} + \mat R_p\mat S\vRv\vct v\vct 1^\top\right ]\\
    &\overset{(d)}{=} c\left[\frac{1}{c} \mat Q_{0:2} \mat U_\text{ray}\mat Z_\text{ray} + \mat R_p\mat S\hat\vmu\vct 1^\top\right ]\\
    &\overset{(e)}{=} \mat Q_{0:2} \mat U_\text{ray}\mat Z_\text{ray} + \vmu\vct 1^\top\\
    &\overset{(f)}{=} \mat T_\text{ray}\mathcal H\left(\mat Z_\text{ray}\right)\,.
\end{align*}

Here, 
(a) follows by unfolding the definition of $\Phi^{-1}$ and simplifying the scalar factors; 
(b) is obtained by using the relation $a\mat O=\frac{1}{c}\mat Z_\text{ray}$ with $\mat Z_\text{ray}$ defined as per main paper;
(c) follows from the fact that we can write $\mat A\mathcal H(\mat X)=\mat A\begin{bmatrix}\mat X^\top& \vct 0\end{bmatrix}^\top + \mat A\vct v\vct 1^\top = \mat A_{0:2}\mat X + \mat A\vct v\vct 1^\top$, where $\vct 1$ is a vector of ones, and $\mat Q\coloneqq \mat R_p\mat S\vRv$ as per definition in~\cref{eq:invPhi}; 
(d) applies the relation $\vRv\vct v=\hat\vmu$; 
(e) results from unfolding the definition of $\hat\vmu$ as per main paper and simplifying matrix/scalar multiplications; 
(f) follows from the relation $\mat A\mat X +\vct y\vct 1^\top = \begin{bmatrix}
    \mat A&\vct y
\end{bmatrix}\mathcal H(\mat X)$ and from the definition of $\mat T_\text{ray}$ as per main paper.

\subsection{Derivation of RayGS's fragment shader}\label{ss:fragment}
We show how we derived the formula used in the fragment shader starting from~\cref{eq:Q_ray}, where $\vx\coloneqq \mat V_\text{ray}\vct\alpha$ for a given interpolating coefficient vector $\vct\alpha$ (\ie nonnegative and summing up to $1$):
\begin{align*}
\fDray(\vx;\vmu,\vE)&=(\tau(\vx)\vx-\vmu)^\top\viE(\tau(\vx)\vx-\vmu)\\
&\overset{(a)}{=}c^2+\tau(\vx)^2\vx^\top\viE\vx-2\tau(\vx)\vx^\top\viE\vmu\\
&\overset{(b)}{=}c^2-\frac{(\vx^\top\viE\vmu)^2}{\vx^\top\viE\vx}\\
&\overset{(c)}{=}c^2-\frac{(\vx^\top\mat R_p\mat S^{-2}\mat R_p^\top\vmu)^2}{\Vert\inv{\mat S}\mat R_p^\top\vx\Vert^2}\\
&\overset{(d)}{=}c^2\left(1-\frac{(\vx^\top\mat R_p\mat S^{-1}\hat\vmu)^2}{\Vert\inv{\mat S}\mat R_p^\top\vx\Vert^2}\right)\\
&\overset{(e)}{=}c^2\left[1-\frac{\left(\vct\alpha^\top\mat V_\text{ray}^\top\mat R_p\inv {\mat S}\hat{\vmu}\right)^2}{\Vert\inv{\mat S}\mat R_p^\top\mat V_\text{ray}\vct\alpha \Vert^2}\right ]\\
\end{align*}
\begin{align*}
\phantom{\fDray(\vx;\vmu,\vE)}&\overset{(f)}{=}c^2\left[1-\frac{\left(\fH\left(\frac{1}{c}\mat U_\text{ray}\mat Z_\text{ray}\vct\alpha\right)^\top\mat Q^\top\mat R_p\inv {\mat S}\hat{\vmu}\right)^2}{\Vert\inv{\mat S}\mat R_p^\top\mat Q\fH\left(\frac{1}{c}\mat U_\text{ray}\mat Z_\text{ray}\vct\alpha\right)\Vert^2}\right ]\\
&\overset{(g)}{=}c^2\left[1-\frac{\left(\fH\left(\frac{1}{c}\mat U_\text{ray}\mat Z_\text{ray}\vct\alpha\right)^\top\mat R_{\hat\vmu\to\vv}\hat{\vmu}\right)^2}{\Vert\mat R_{\hat\vmu\to\vv}\fH\left(\frac{1}{c}\mat U_\text{ray}\mat Z_\text{ray}\vct\alpha\right)\Vert^2}\right ]\\
&\overset{(h)}{=}c^2\left[1-\frac{\left(\fH\left(\frac{1}{c}\mat U_\text{ray}\mat Z_\text{ray}\vct\alpha\right)^\top\vct v\right)^2}{\Vert\fH\left(\frac{1}{c}\mat U_\text{ray}\mat Z_\text{ray}\vct\alpha\right)\Vert^2}\right ]\\
&\overset{(i)}{=}c^2\left[1-\frac{1}{1+c^{-2}\Vert\mat Z_\text{ray}\vct\alpha\Vert^2}\right ]\\
&\overset{(j)}{=}\inv{\left[c^{-2}+\Vert\mat Z_\text{ray}\vct\alpha\Vert^{-2}\right]}\,.
\end{align*}
Here, 
(a) follows by simple algebraic manipulations and by considering $c\coloneqq\sqrt{\vmu^\top\viE\vmu}$ as per main paper; 
(b) follows by unfolding the definition of $\tau$ as per main paper and by simple algebraic manipulations; 
(c) follows by unfolding $\viE\coloneqq \mat R_p\mat S^{-2}\mat R_p$ and rewriting the denominator into a squared norm; 
(d) follows by substituting $\mat S^{-1}\mat R_p^\top\vmu=c\hat\vmu$, where $\hat\vmu$ is as per main paper, and factorizing; 
(e) follows by unfolding the definition of $\vx$ provided above; 
(f) follows by unfolding $\mat V_\text{ray}$ using the relation (b) in~\cref{ss:V_ray}, using the identity $\mathcal H(\mat Z_\text{ray})\vct\alpha =\mathcal H(\mat Z_\text{ray}\vct\alpha)$, and simplifying scalar factors; 
(g) follows by unfolding the definition of $\mat Q$ as per main paper and simplifying matrix multiplications; 
(h) follows by noting that the norm of a rotated vector yields the norm of the vector and that $\mat R_{\hat\vmu\to\vv}\hat{\vmu}=\vct v$; 
(i) follows by from the identities $\fH(\vct z)^\top\vct v=1$ and $\Vert \fH(x\vct z)\Vert^2=1+x^2\Vert \vct z\Vert^2$, and that the norm is invariant to rotations of the argument. 
Finally, (j) follows by rearranging and simplifying terms.

\subsection{Derivation of $P_\text{MIP}(\vct x)$ in~\cref{sec:mip}}\label{ss:P}
We start by approximating the 3D-2D projection operator $\pi(\vct x)$ to the first-order around $\vct x_0$, yielding
\[
\hat\pi(\vct x)\coloneqq \pi(\vct x_0)+J_\pi(\vct x_0)(\vct x-\vct x_0)\,.
\]
Assuming $\vct x\sim\mathcal N(\vct \mu,\mat \Sigma)$, we have that $\vct u\coloneqq\hat\pi(\vct x)\sim\mathcal N(\hat\pi(\vct \mu),\mat J\mat\Sigma\mat J^\top)$, where $\mat J\coloneqq J_\pi(\vct x_0)$.
For each (unit) camera ray $\vct x$, we take $\vct x_0\coloneqq \tau(\vct x)\vct x$, where $\tau$ is as per main paper, and define a 2D Gaussian distribution $\mathcal N(\pi(\vct x_0),\sigma_{x}^2)$. The density of rays $\vct x$ is then determined by the expectation of $\mathcal N(\vct u;\pi(\vct \mu),\mat J\mat\Sigma\mat J^\top)$ with $\vct u\sim\mathcal N(\pi(\vct x_0),\sigma_{\vct x}^2\mat I)$, which yields
\[
\begin{aligned}
P_\text{MIP}(\vct x)&\coloneqq \int_{\mathbb R^2} \mathcal N(\vct u;\pi(\vct \mu),\mat J\mat\Sigma\mat J^\top)\mathcal N(\vct u;\pi(\vct x_0),\sigma_{\vct x}^2\mat I)d\vct u\\
&=\mathcal N(\pi(\vct x_0);\pi(\vct \mu),\mat {J\Sigma J^\top}+\sigma_{\vct d}^2\mat I)\,.
\end{aligned}
\]

We consider a projection operator that varies with the viewing ray $\vct x$. Specifically, we project 3D points to the spherical tangent space $\mathcal T_x\subset \mathbb R^3$ of $\vct x$ by leveraging the logarithmic map. The 3D tangent vector can be mapped to 2D points on the tangent plane by leveraging $\emph{any}$ local coordinate frame expressed as unitary, orthogonal columns of $\mat F_x\in\mathbb R^{3\times 2}$, which satisfies $\mat F_{\vct x}^\top\mat F_{\vct x}=\mat I$ and $\mat F_{\vct x}^\top \vct x=\vct 0$. This yields the following projection operator:
\[
\pi_{\vct x}(\vct z)\coloneqq \mat F_{\vct x}^\top\text{Log}_{\vct x}\left (\frac{\vct z}{\|\vct z\|}\right)\,.
\]
Then, $\pi_{\vct x}(\vct x_0)=\vct 0$ and the Jacobian $\mat J_{\vct x}\coloneqq J_{\pi_{\vct x}}(\vct x_0)$ of this projection operator evaluated at $\vct x_0$ is given by $\mat J_{\vct x}=\frac{\mat F_{\vct x}^\top}{\tau(\vct x)}$. It follows that
\[
\hat \pi_{\vct x}(\vct z)=\frac{1}{\tau(\vct x)}\mat F_{\vct x}^\top \vct z\,,
\]
which yields the following projected and smoothed 2D Gaussian distribution
\[
\begin{aligned}
P_\text{MIP}(\vct x)&=\mathcal N\left (\vct 0;\mat F_{\vct x}^\top \vct \mu;\mat F_{\vct x}^\top\mat\Sigma\mat F_{\vct x}+(\tau(\vct x)\sigma_{\vct x})^2\mat I\right )\tau(\vct x)^2\\
&=\mathcal N\left (\vct 0;\mat F_{\vct x}^\top \vct \mu;\mat F_{\vct x}^\top\left[\mat\Sigma+(\tau(\vct x)\sigma_{\vct x})^2\mat I\right]\mat F_{\vct x}\right )\tau(\vct x)^2\\
&\coloneqq\mathcal N\left (\vct 0;\mat F_{\vct x}^\top \vct \mu;\mat F_{\vct x}^\top\hat{\mat\Sigma}_x\mat F_{\vct x}\right )\tau(\vct x)^2\,,
\end{aligned}
\]
where we set $\hat{\mat\Sigma}_x\coloneqq\mat\Sigma+(\tau(\vct x)\sigma_{\vct x})^2\mat I$.
To evaluate the Gaussian distribution on a given view and ray $\vct x$ we need to compute $\vct\mu^\top\mat F_{\vct x}(\mat F_{\vct x}^\top\hat{\mat \Sigma}_x\mat F_{\vct x})^{-1}\mat F_{\vct x}^\top\vct \mu$ and $\det(\mat F_{\vct x}^\top\hat{\mat \Sigma}_x\mat F_{\vct x})$. In order to get rid of the dependency on $\mat F_{\vct x}$ we can rewrite those expressions as follows:
\[
\begin{aligned}
\vct\mu^\top&\mat F_{\vct x}(\mat F_{\vct x}^\top\hat{\mat \Sigma}_x\mat F_{\vct x})^{-1}\mat F_{\vct x}^\top\vct \mu\\
&\stackrel{(a)}{=}\vct\mu^\top\mat F_{\vct x}\mat F_{\vct x}^\top\left[\hat{\mat\Sigma}_x^{-1}-\frac{\hat{\mat\Sigma}_x^{-1}\vct x\vct x^\top\hat{\mat\Sigma}_x^{-1}}{\vct x^\top\hat{\mat\Sigma}_x^{-1}\vct x}\right]\mat F_{\vct x}\mat F_{\vct x}^\top\vct\mu\\
&\stackrel{(b)}{=}\vct\mu^\top(\mat I-\vct x\vct x^\top)\left[\hat{\mat\Sigma}_x^{-1}-\frac{\hat{\mat\Sigma}_x^{-1}\vct x\vct x^\top\hat{\mat\Sigma}_x^{-1}}{\vct x^\top\hat{\mat\Sigma}_x^{-1}\vct x}\right](\mat I-\vct x\vct x^\top)\vct\mu\\
&\stackrel{(c)}{=}\vct\mu^\top\left[\hat{\mat\Sigma}_x^{-1}-\frac{\hat{\mat\Sigma}_x^{-1}\vct x\vct x^\top\hat{\mat\Sigma}_x^{-1}}{\vct x^\top\hat{\mat\Sigma}_x^{-1}\vct x}\right]\vct\mu\\
&\stackrel{(d)}{=}(\vct\mu-\hat\tau(\vct x)\vct x)^\top\hat{\mat\Sigma}_x^{-1}(\vct\mu-\hat\tau(\vct x)\vct x)=\fDray(\vx;\vmu,\hat{\vE}_x)\,,
\end{aligned}
\]
where $\hat \tau(\vct x)\coloneqq\frac{\vct\mu^\top\hat{\mat \Sigma}_x^{-1}\vct x}{\vct x^\top\hat{\mat \Sigma}_x^{-1}\vct x}$, and
\[
\begin{aligned}
\det(\mat F_{\vct x}^\top\hat{\mat \Sigma}_x\mat F_{\vct x})&\stackrel{(a)}{=}\det(\hat{\mat\Sigma}_x)\vct x^\top\hat{\mat\Sigma}_x^{-1}\vct x\,.
\end{aligned}
\]
Equalities (a) follow from \Cref{prop:inverse-block-matrix} by taking $\mat X\coloneqq\mat W^\top\hat{\mat \Sigma}\mat W$, where $\mat W=[\mat F_{\vct x}, \vct x]$, and considering $\mat A\coloneqq\mat F_{\vct x}^\top \hat{\mat \Sigma}_x\mat F_{\vct x}$.
Since $\mat W$ is unitary, we have that $\mat X^{-1}=\mat W^\top\hat{\mat \Sigma}_x^{-1}\mat W$ and $\det(\mat X)=\det(\hat{\mat \Sigma}_x)$. It follows that $\mat D=\mat F_{\vct x}^\top \hat {\mat\Sigma}_x^{-1}\mat F_{\vct x}$, $\vct e=\mat F_{\vct x}^\top \hat {\mat\Sigma}_x^{-1} \vct x$ and $f=\vct x^\top \hat {\mat\Sigma}_x^{-1} \vct x$. Moreover we have that $\vct b=\mat F_{\vct x}^\top \hat {\mat\Sigma}_x \vct x$ and $c=\vct x^\top \hat {\mat\Sigma}_x \vct x$.
Equality (b) follows by noting that columns of $\mat W$ are eigenvectors of $\mat I-\vct x\vct x^\top$ with eigenvalues given by $[1,1,0]$ and therefore $\mat I-\vct x\vct x^\top=\mat F_{\vct x}\mat F_{\vct x}^\top$. Finally, equalities (c) and (d) follow by simple algebraic manipulations.

The final form of $P_\text{MIP}(\vct x)$, up to constant multiplicative factors, is given by
\[
P_\text{MIP}(\vct x)\propto\frac{\tau(\vct x)^2}{\sqrt{\det(\hat{\mat\Sigma}_x)\vct x^\top\hat{\mat \Sigma}_x^{-1}\vct x}}\exp\left(-\frac{1}{2}\fDray(\vx;\vmu,\hat{\vE}_x)\right)\,.
\]

\subsection{Useful results}
\begin{proposition}\label{prop:inverse-block-matrix}
    Assume
    \[
    \mat X\coloneqq\begin{bmatrix}
    \mat A&\vct b\\
    \vct b^\top&c
    \end{bmatrix}\qquad\text{and}\qquad\mat X^{-1}\coloneqq    \begin{bmatrix}
    \mat D&\vct e\\
    \vct e^\top&f
    \end{bmatrix}\,.
    \]
    Then $\mat A^{-1}=\mat D-\frac{\vct e\vct e^\top}{f}$ and $\det(\mat A)=\frac{\det(X)}{c-\vct b^\top\mat A^{-1}\vct b}$.
\end{proposition}
\begin{proof}
    Since $\mat X\mat X^{-1}=\mat I$ we have that $\mat A\mat D+\vct b\vct e^\top=\mat I$ and therefore $\mat D=\mat A^{-1}-\mat A^{-1}\vct b\vct e^\top$. 
    We have also that $\mat A\vct e+\vct b f=\vct 0$ and therefore $\vct e/f=-\mat A^{-1}\vct b$. Hence by substituting we have $\mat D=\mat A^{-1}+\vct e\vct e^\top/f$, from which the result about the inverse of $\mat A$ follows. 

    By properties of the determinant we have that $\det(\mat X)=\det(\mat A)(c-\vct b^\top \mat A^{-1}\vct b)$, from which the result about the determinant of $\mat A$ trivially follows.
\end{proof}
\section{Additional quantitative experiments}\label{sec:quantitative}
\subsection{GS vs VKGS}\label{ss:GSvsVKGS}
In~\cref{tab:quantitative_gs}, we report the results obtained by GS versus the Vulkan counterpart VKGS. We have split the table into two sections to distinguish scenes from MipNerf360 (top) and Tanks\&Temples (bottom). For each scene, we report the size of the model in terms of number of primitives and report left-to-right speed comparisons in terms of FPS and quality metrics in terms of PSNR, SSIM and LPIPS, averaged over all the test views. We disregard by now the last two columns in the table.
Staring from the rendering speed, it is crystal clear that the Vulkan implementation outclasses the CUDA-based from GS with $2\times$ average speedup, despite GS including the most recent optimizations in the renderer. Except for a few cases, the quality metrics are not significantly different. However, the fact that there are differences
indicates potential misalignment between the implementations and the results favor GS because the model has been trained with the same renderer. We noticed that the bigger discrepancies happen on models that suffer from a lot of big semi-transparent floater.
\begin{table}[thb]
    \centering
    \setlength{\tabcolsep}{2pt}
    \resizebox{\columnwidth}{!}{
    \begin{tabular}{c||c||cc||cc|cc|cc}\toprule
        {\footnotesize on RTX2080}&&\multicolumn{2}{c}{FPS $\uparrow$ } &\multicolumn{2}{c}{PSNR $\uparrow$}&\multicolumn{2}{c}{SSIM $\uparrow$}&\multicolumn{2}{c}{LPIPS $\downarrow$}\\
        Scene&$N$&GS&VKGS&GS&VKGS&GS&VKGS&GS&VKGS\\\midrule\midrule
        bicycle&6.13M&52&\bf 144&25.24&24.99&0.768&0.752&0.229&0.243\\
bonsai&1.24M&105&\bf 249&31.98&31.47&0.938&0.926&0.253&0.221\\
counter&1.22M&82&\bf 193&28.69&28.57&0.905&0.894&0.262&0.233\\
flowers&3.64M&95&\bf 171&21.52&21.45&0.600&0.594&0.366&0.367\\
garden&5.83M&58&\bf 126&27.41&26.98&0.867&0.852&0.119&0.127\\
kitchen&1.85M&65&\bf 151&30.32&30.10&0.921&0.908&0.158&0.156\\
room&1.59M&83&\bf 223&30.63&30.70&0.913&0.900&0.289&0.260\\
stump&4.96M&75&\bf 168&26.55&26.25&0.772&0.761&0.244&0.253\\
treehill&3.78M&81&\bf 169&22.49&22.41&0.634&0.621&0.367&0.37\\
\midrule
barn&0.89M&129&\bf 216&29.06&26.02&0.880&0.860&0.201&0.238\\
caterpillar&1.07M&129&\bf 212&24.36&22.16&0.824&0.796&0.234&0.264\\
ignatius&2.25M&108&\bf 186&22.21&21.10&0.823&0.797&0.187&0.210\\
meetingroom&1.05M&122&\bf 258&26.23&23.15&0.893&0.854&0.209&0.239\\
truck&2.54M&105&\bf 170&25.19&24.21&0.876&0.852&0.178&0.168\\
\bottomrule
    \end{tabular}
    }
    
    \caption{Speed and quality comparison between GS and VKGS on scenes from MipNeRF360 and Tanks\&Temples. The last two columns provide speed comparison of VKGS versus VKRayGS on MipNeRF360 using the same GS scene model.}
    \label{tab:quantitative_gs}
\end{table}

\section{Qualitative examples}\label{sec:qualitative}
In~\cref{fig:qualitative_1,fig:qualitative_2,fig:qualitative_3} we provide the rendering of the first test image of all scenes we evaluated on. We report results obtained with GOF and our VKRayGS. The goal is to validate that visually there is barely any visible difference between the two renderings, while our method being order of magnitude faster. Clearly, given that the quality scores do not perfectly match, there are some differences that are mainly due to discrepancies between the Vulkan-based rendering pipeline and the one from GOF. \Eg we might have different ways of culling primitives that introduce differences mainly at the borders, or different ways of addressing out-of-bound colors, which could make saturated areas darker. Aligning differences between the Vulkan-based renderer and GOF are beyond our contributions, and thus out of our paper's scope.
\begin{figure*}
    \centering
    \includegraphics[width=.75\linewidth]{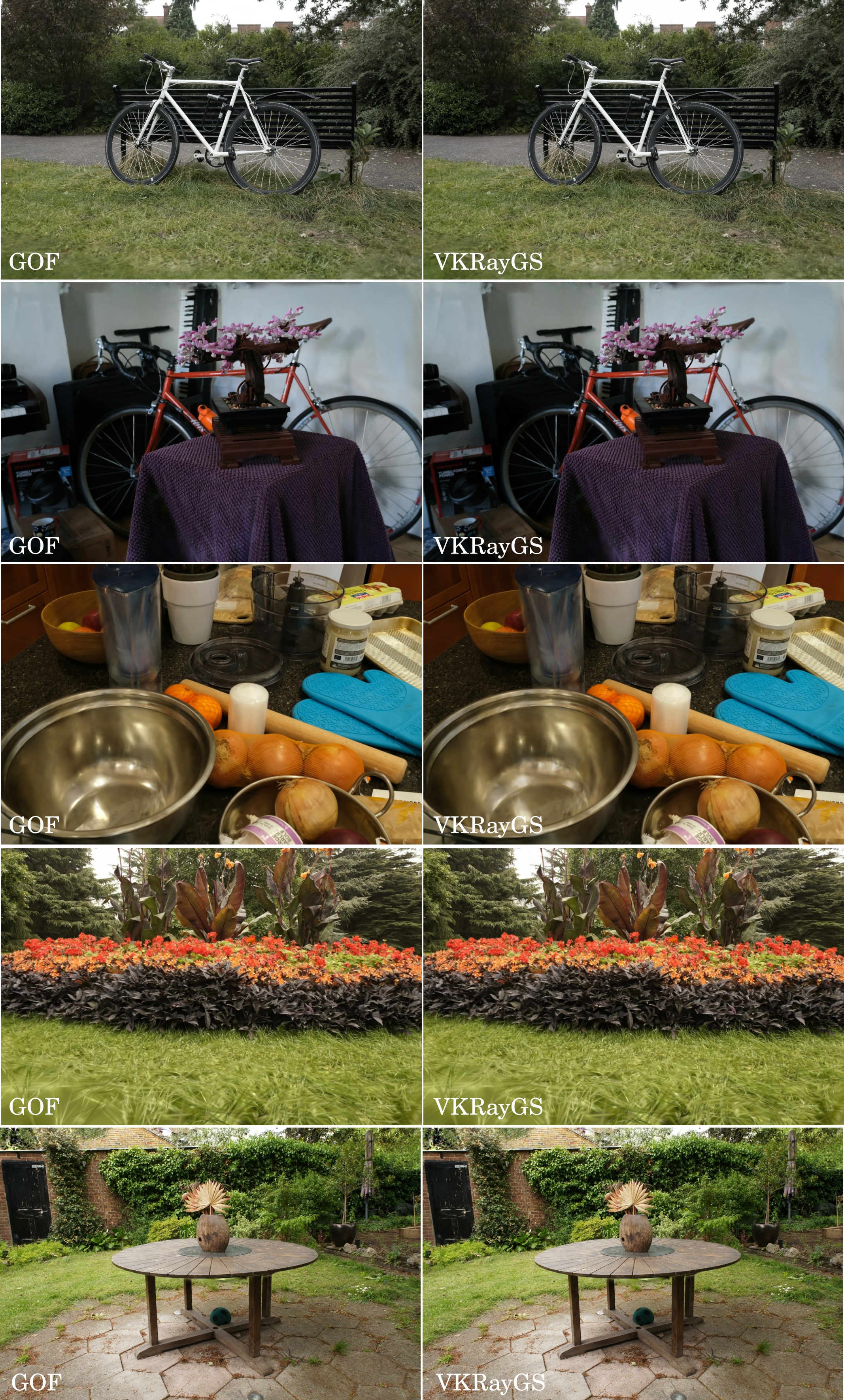}
    \caption{First test images of the MipNerf360 scenes \emph{bicycle}, \emph{bonsai}, \emph{counter}, \emph{flowers} and \emph{garden}, rendered by GOF and our method VKRayGS.}
    \label{fig:qualitative_1}
\end{figure*}
\begin{figure*}
    \centering
    \includegraphics[width=.75\linewidth]{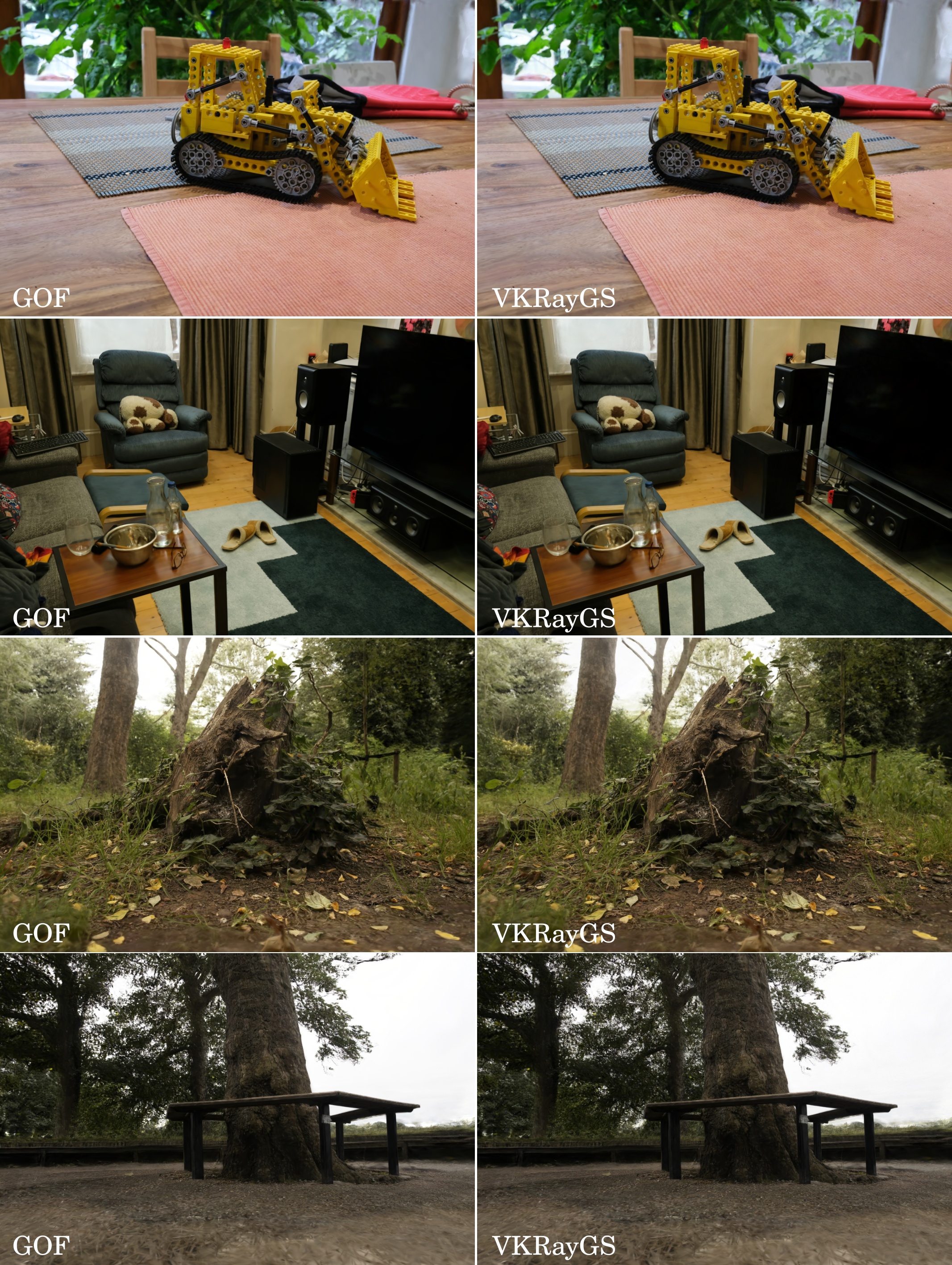}
    \caption{First test images of the MipNerf360 scenes \emph{kitchen}, \emph{room}, \emph{stump} and \emph{treehill}, rendered by GOF and our method VKRayGS.}
    \label{fig:qualitative_2}
\end{figure*}
\begin{figure*}
    \centering
    \includegraphics[width=.75\linewidth]{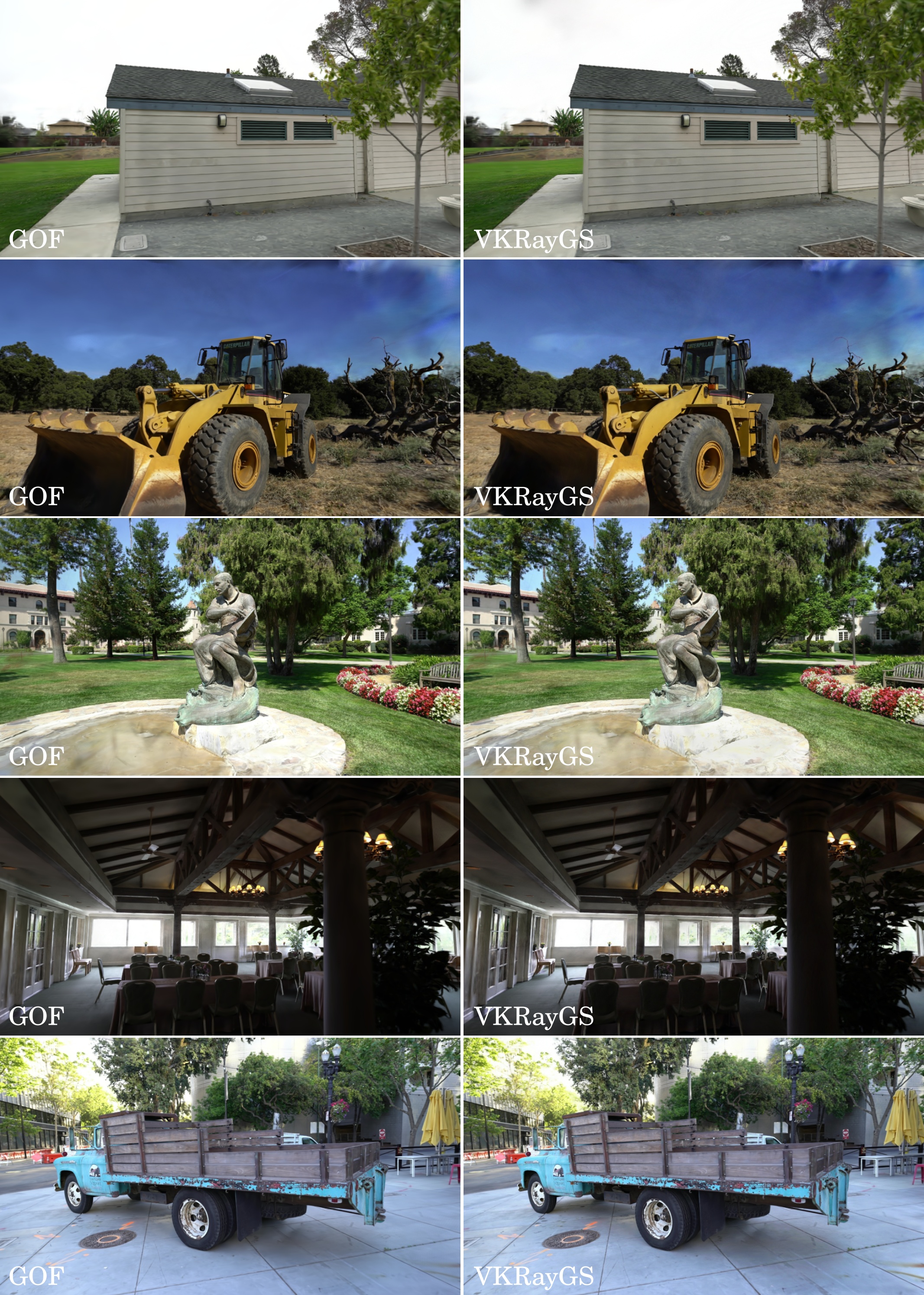}
    \caption{First test images of the Tanks\&Temples scenes \emph{barn}, \emph{caterpillar}, \emph{ignatius}, \emph{meetingroom} and \emph{truck}, rendered by GOF and our method VKRayGS.}
    \label{fig:qualitative_3}
\end{figure*}

\end{document}